\documentclass[letterpaper, 10 pt, conference]{ieeeconf}  

\IEEEoverridecommandlockouts                              
\title{\LARGE \bf
Incentivized Lipschitz Bandits
}

\author{Sourav Chakraborty$^*$, Amit Kiran Rege$^*$, Claire Monteleoni, Lijun Chen
\thanks{$^*$Equal Contribution}
\thanks{All authors are with the Department of Computer Science,
        University of Colorado, Boulder, CO 80309, USA
        {\tt\footnotesize \{sourav.chakraborty, amit.rege, cmontel, lijun.chen\}@colorado.edu}}%
\thanks{Claire Monteleoni is also with INRIA, Paris, France.}
}

\usepackage{subcaption}
\usepackage{svg}
\usepackage{caption}
\usepackage{amsmath}
\usepackage{hyperref}       
\usepackage{url}            
\usepackage{booktabs}       
\usepackage{amsfonts}       
\usepackage{nicefrac}       
\usepackage{microtype}      

\usepackage{algorithm}

\usepackage{amssymb}		
\usepackage{graphicx}		

\usepackage{algorithmic}
\usepackage{comment}
\usepackage{wrapfig}
\usepackage{multirow}
\usepackage{tocloft}
\usepackage{tikz}

\newtheorem{theorem}{Theorem}
\newtheorem{lemma}{Lemma}

\newtheorem{assumption}{Assumption}

\begin{document}

\maketitle
\thispagestyle{empty}
\pagestyle{empty}

\begin{abstract}
We study incentivized exploration in multi-armed bandit (MAB) settings with infinitely many arms modeled as elements in continuous metric spaces. Unlike classical bandit models, we consider scenarios where the decision-maker (principal) incentivizes myopic agents to explore beyond their greedy choices through compensation, but with the complication of reward drift—biased feedback arising due to the incentives. We propose novel incentivized exploration algorithms that discretize the infinite arm space uniformly and demonstrate that these algorithms simultaneously achieve sublinear cumulative regret and sublinear total compensation. Specifically, we derive regret and compensation bounds of $\Tilde{O}(T^{d+1/d+2})$, with  $d$ representing the covering dimension of the metric space. Furthermore, we generalize our results to contextual bandits, achieving comparable performance guarantees. We validate our theoretical findings through numerical simulations.
\end{abstract}

\section{Introduction} \label{sec:intro}
The multi-armed bandit (MAB) problem is a framework \cite{slivkins2021introduction, lattimore_2020} for sequential decision-making under uncertainty, with applications in search engines (\cite{search-sys}), clinical trials (\cite{gittins1, berry, william}), recommendation systems (\cite{bouneffouf:hal-00753401, Li_2010}), financial portfolio design (\cite{brochu}), and cognitive radio networks (\cite{cogradio}). In the classical stochastic MAB setting, a decision-maker selects an arm at each time step, receives a reward, and uses this feedback to inform future choices. The objective is to minimize the cumulative regret, defined as the difference between the total reward obtained by always selecting the optimal arm and the actual reward accrued by the algorithm. Achieving this goal necessitates a balance between leveraging current knowledge (i.e., exploitation) and gathering new information (i.e., exploration). Over-exploitation may preclude the discovery of superior alternatives, whereas excessive exploration limits immediate reward accumulation.

A key assumption in traditional MAB formulations is that the decision-maker (the \emph{principal}) both selects and pulls the arm. However, in many real-world scenarios, these roles are distinct. The principal may delegate decision-making to \emph{agents} whose objectives are misaligned with the principal’s long-term interests. Specifically, while the principal seeks to optimize long-term performance through exploration, agents may be myopic, favoring immediate rewards and thus engaging in pure exploitation. This misalignment of objectives calls for strategies that effectively align agent incentives with the principal’s exploratory objectives.

For instance, consider an e-commerce platform such as Amazon. The platform (principal) seeks to maximize long-term revenue by encouraging customers (agents) to explore a diverse range of products (arms) to identify the most lucrative offerings. However, customers naturally tend to select products with the highest current ratings, leading to limited exploration and potential suboptimal long-term outcomes \cite{bubeck2012regret, suttonbarto}. Similar challenges arise in content recommendation platforms such as Netflix, where viewers may repeatedly select well-rated but suboptimal content, impeding algorithmic optimization.
\begin{figure}[h]
    \centering
    \includegraphics[width=0.4\textwidth]{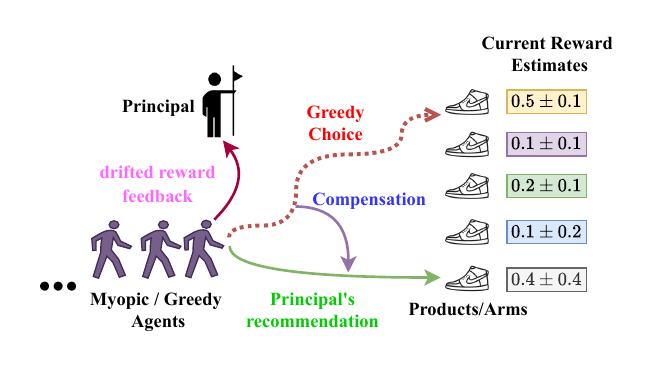}
    \caption{An illustration of the Incentivized Exploration setup with a Principal and a stream of agents.}
    \label{fig:ic}
\end{figure}
To address such misaligned incentives, \emph{incentivized exploration} has emerged as a promising approach \cite{fraz, mansour, wang, immorlica2019bayesian}. Here, the principal provides compensation to agents who select arms other than their empirically best option, thereby fostering exploration while managing the associated costs. Early works on incentivized MAB models \cite{immorlica2019bayesian, wang, Hirnschall_Singla_Tschiatschek_Krause_2018, Han_2015, Liu_Ho_2018} assumed unbiased agent feedback. However, empirical studies suggest that incentives such as coupons or monetary rewards often introduce biases, inflating reported satisfaction levels and distorting the reward structure \cite{martensen, Ehsani2015EffectOQ}. This phenomenon, termed \emph{reward drift}, complicates the exploration-exploitation trade-off, as inflated rewards can erroneously signal suboptimal arms as optimal. Recent work by \cite{liu20} has addressed this issue, demonstrating that incentivized exploration remains effective even in the presence of biased feedback.

While these approaches are effective for small, finite arm sets, they become impractical as the number of arms grows extremely large—often warranting a formalization as an infinite space. This challenge is increasingly relevant in modern applications, where platforms like Amazon manage millions of products, and streaming services such as Netflix curate vast content libraries. This naturally raises the central question that underpins this work: \textit{Can we achieve effective incentivized exploration in settings with an infinite number of arms, particularly when reward feedback is subject to drift?}\\

\noindent\textbf{Contributions.} In this work, we extend the incentivized exploration framework \cite{wang,liu20}, illustrated in Fig.~\ref{fig:ic}, to settings with an \emph{infinite number of arms}, where arms are conceptualized as elements in a continuous metric space.  
In this setting, an agent’s compensation from the principal is determined by the difference in estimated rewards between the principal’s recommended arm and the agent’s myopic (greedy) choice. However, the agent’s feedback is biased due to a reward drift term that increases with compensation, distorting the learning process.

To address these challenges, we make the following key contributions:  
\textbf{(1)} We propose \emph{incentivized exploration algorithms} that guarantee sublinear regret and sublinear compensation \emph{simultaneously} in infinite-armed bandits.  
\textbf{(2)} We analyze a \emph{uniform discretization} approach, where the arm space is discretized at a fixed resolution, and prove that an appropriately tuned UCB-based algorithm achieves regret and compensation bounds of \( \Tilde{O} ( T^{(d+1)/(d+2)} ) \), where \( d \) is the covering dimension of the metric space. \textbf{(3)} We extend our framework to \emph{contextual bandits}, where rewards depend on both the arm and an observed context, and demonstrate that a structured discretization of the arm-context space yields regret and compensation bounds of \( \Tilde{O} (T^{(d + 1)/(d + 2)}) \), where \( d = d_x + d_a \) and \( d_x \) and \( d_a \) denote the covering dimensions of the context and arm spaces, respectively. \textbf{(4)} Finally, we support our theoretical findings with empirical results demonstrating that our algorithms achieve significant improvements in both regret and compensation efficiency over naive baselines. Proofs are provided in the appendix (\cite{appx}).\\

\noindent\textbf{Related Work.}  Early works by \cite{fraz}, \cite{kremer}, and \cite{che} introduced Bayesian models for incentivized exploration, focusing on minimizing both regret and compensation in discounted settings. \cite{mansour} extended these results to the non-discounted case, proposing an algorithm with \(O(\sqrt{T})\) regret. Building on this, \cite{wang} analyzed non-Bayesian and non-discounted environments, achieving \(O(\log T)\) regret and compensation guarantees. 

Further advancements were made by \cite{liu20}, who addressed the challenge of biased feedback induced by incentives, demonstrating that their algorithm maintains \(O(\log T)\) regret even under reward drift. More recently, \cite{chakraborty2024incentivized} generalized this setting to non-stationary rewards, establishing sublinear bounds on both regret and compensation. Our work builds upon the framework introduced by \cite{liu20}, extending it to settings where arms form an infinite metric space. Additionally, we generalize the incentivization mechanisms and analyze their impact on regret rates in this more complex environment.

Robustness to adversarial interference in MABs has also been extensively studied. \cite{lykouris2018stochastic} proposed a multi-layer elimination algorithm to mitigate adversarial corruptions in stochastic bandits, while \cite{feng2020intrinsic} explored rational arm behavior, establishing logarithmic regret bounds. Research on Bayesian Incentive Compatible (BIC) models, such as those by \cite{mansour2019bayesian}, \cite{cohen}, and \cite{sellke2021price}, examines how principals can strategically encourage beneficial agent behavior, closely linking to Bayesian Persuasion \cite{kamenica}. For a comprehensive overview of these methods, see \cite{slivkins2021introduction}.

The infinite-armed bandit setting has been studied in various contexts \cite{Kleinberg-lipschitz}. However, our work is the first to investigate this setting in the realm of incentivized exploration.

\section{Preliminaries}
\noindent\textbf{Stochastic Bandits.} In the standard stochastic bandit framework, a decision maker selects an arm $a$ from a set of $K$ arms at each discrete time step $t \in \{1, 2, \dots, T\}$, based on the historical sequence of arm selections and observed rewards. The reward $\rho_t(a)$ associated with each arm $a$ is assumed to follow an i.i.d. random distribution $\mathcal{D}(a)$ with mean $\mu(a)$. It is further assumed that rewards lie within the interval $[0,1]$. Let $a^*$ denote the optimal arm, which yields the highest expected reward. The regret $\mathcal{R}_T$ incurred by an algorithm is defined as follows:
\begin{align*}
    \mathcal{R}_T \triangleq T \cdot \mu(a^*) - \sum_{t=1}^T \rho_t(a_t),
\end{align*}
where $a_t$ is the arm chosen by the algorithm at time $t$. The objective is to devise algorithms that achieve sublinear expected regret as a function of $T$. Common algorithms that achieve this objective include UCB1 (\cite{auer2002finite, lai1985asymptotically}), Thompson Sampling (\cite{russo2018tutorial}), and the $\epsilon$-Greedy algorithm (\cite{auer2002finite, suttonbarto}).\\

\noindent\textbf{Incentivized Exploration.} In practical settings, the entity selecting arms (the principal) and the entity executing the selection (the agent) may differ, leading to potential misalignment of interests. Agents tend to prioritize exploitation, selecting arms with the highest observed empirical rewards, whereas the principal seeks to balance exploration and exploitation. To mitigate this, the principal provides compensation to incentivize exploration \cite{wang}. At each time step \(t\), the agent selects an action \(a_t\) based on the principal's recommendation and receives compensation \(\kappa_t\), after which the agent observes a reward \(r_t\) and reports it to the principal.

The principal can use a bandit algorithm such as UCB1 or \(\epsilon\)-Greedy to balance exploration and exploitation while minimizing long-term regret. Compensation \(\kappa_t\) is determined by the difference between the empirical rewards of the agent's greedy choice \(g_t\) and the principal’s recommended arm \(a_t\). However, this compensation may introduce bias in the agent’s feedback, resulting in a distorted observed reward $r_t = \mathcal{F}(\rho_t(a_t), \kappa_t)$ instead of the true reward \(\rho_t(a_t)\). Following \cite{liu20}, we model this effect using an additive drift framework, where the observed reward is $r_t = \rho_t(a_t) + \gamma_t(\kappa_t)$, with \(\gamma_t(\cdot)\) being a non-decreasing function of \(\kappa_t\). To ensure stability, we impose the following constraint:  
\begin{assumption} \label{drift-assm}
The reward drift function \(\gamma_t(a)\) is non-decreasing, satisfies \(\gamma_t(0) = 0\), and is Lipschitz continuous, i.e., there exists a constant \(\ell_t\) such that  
\[
    |\gamma_t(a) - \gamma_t(a')| \leq \ell_t |a - a'|, \quad \forall a, a'.
\]
\end{assumption}
The empirical reward $\hat{\mu}_t(a)$ is thus the average of these biased rewards. In addition to minimizing regret, the principal aims to minimize the total compensation given:
\[
\mathcal{C}_T \triangleq \sum_{t=1}^T \kappa_t = \sum_{t=1}^T \left(\hat{\mu}_t(g_t) - \hat{\mu}_t(a_t)\right).
\]
We evaluate the performance of incentivized exploration strategies based on the expected regret and expected compensation, aiming to design algorithms that achieve sublinear regret and compensation simultaneously.

\section{Problem Formulations} \label{prob-form}
\noindent\textbf{Incentivized Exploration with Stochastic Rewards.} \label{stoch-prob-form} We consider a variant of the multi-armed bandit problem with an uncountable set of arms $A$, structured as a metric space $(A, \Phi)$, where the metric $\Phi$ defines the distance between arms. A natural example is $A \subseteq [0,1]^d$, a subset of the $d$-dimensional Euclidean space equipped with the $\ell_2$ metric, denoted $(A, \ell_2)$. Each arm $a \in A$ corresponds to a $d$-dimensional vector, and its reward is independently drawn from a probability distribution supported on $[0,1]$ with mean $\mu(a)$. We denote such a problem instance by $(A, \Phi, \mu)$.

Solving this problem efficiently is generally infeasible without additional assumptions on the reward function \cite{Kleinberg-lipschitz}. To address this, we impose a Lipschitz condition, ensuring that similar arms yield similar expected rewards:
\begin{equation} \label{lips}   
    |\mu(a) - \mu(a')| \leq L \cdot \Phi(a, a'), \quad \forall a,a' \in A,
\end{equation}
where $L$ is a Lipschitz constant. We analyze this problem in the context of incentivized exploration \cite{wang, liu20}, a paradigm relevant to large-scale platforms such as Amazon and Netflix, where the platform (principal) strategically encourages users (agents) to explore various products (arms). Given the high-dimensional nature of such platforms, modeling arms as elements in a metric space allows for an effective representation of structural dependencies, such as categorical similarities and user preferences.

To ensure computational feasibility, we discretize $A$ into a finite subset $A_0 \subset A$. Define \(\mu_0^* = \sup_{a_0 \in A_0} \mu(a_0)\)
as the maximum expected reward attainable within $A_0$, and introduce the approximation error \(\delta(A_0) = \mu^* - \mu_0^*\), where $\mu^* = \sup_{a \in A} \mu(a)$ denotes the global optimal reward. Let $\mathcal{R}_T(A)$ be the cumulative regret of a given learning algorithm over a time horizon $T$, given by:
\begin{align} \label{our-reg-form}
    \mathbb{E}[\mathcal{R}_T(A)] = \mathbb{E} \left[ \sum_{t=0}^{T} \left(\mu^* - \rho_t(a_t) \right) \right].
\end{align}
The objective is to minimize expected regret by optimizing the selection of the discrete approximation set $A_0$. To approximate $A$ efficiently, we construct $A_0$ using a $\psi$-covering of $A$ which consists of subsets $S_i \subset A$ satisfying:
\[
\operatorname{diam}(S_i) \leq \psi, \quad \bigcup_i S_i = A.
\]
A representative arm is selected from each subset to form $A_0$. We refer the reader to the appendix (\cite{appx}) for more information regarding metric spaces and covering numbers. To further encourage exploration, the principal introduces a compensation \(\kappa_t\), modifying the observed reward to $r_t = \rho_t(a_t) + \gamma_t(\kappa_t),$ where \(\gamma_t(\cdot)\) is an increasing function modeling the agent's response to incentives. The compensation is given by  
\begin{equation}
    \kappa_t = \hat{\mu}_t(g_t) - \hat{\mu}_t(a_t),
\end{equation}
where \(g_t\) is the agent’s greedy choice based on empirical estimates \(\hat{\mu}_t(a)\). The principal's objective is to minimize both regret and total incentive cost, given by $\mathcal{C}_T(A) = \sum_{t=1}^T \kappa_t$.
The following sections establish sublinear bounds on both regret and compensation, leveraging covering dimension \(d\), Lipschitz constant \(L\), and discretization scale \(\psi\).\\

\noindent\textbf{Incentivized Exploration with Contextual Information.} \label{contextual-setting} Building upon the incentivized exploration framework introduced, we extend the setting to a more general and practically relevant class of problems known as contextual bandits. In this extension, the principal observes a context before selecting an arm, and the expected reward depends jointly on both the context and the chosen arm. This model is particularly suited to large-scale decision-making scenarios where reward distributions are influenced by external factors. Examples include personalized recommendation systems where user profiles determine product relevance, or adaptive pricing strategies where market conditions affect optimal pricing decisions.

We consider a contextual bandit problem where the space of contexts and arms are modeled as metric spaces. Specifically, let the context space be represented as \((X, \Phi_X)\), where \(X \subseteq [0,1]^{d_x}\) is equipped with a metric \(\Phi_X\), and let the arm space be given by \((A, \Phi_A)\), where \(A \subseteq [0,1]^{d_a}\) is equipped with a metric \(\Phi_A\). The reward of an arm $a \in A$ is independently drawn from a probability distribution, supported on $[0,1]$, with expectation $\nu(a, x)$ for a particular context $x \in X$. We denote such a problem instance by $(A, X, \Phi_X, \Phi_A, \nu)$.

We assume that the expected reward function \(\nu: A \times X \to \mathbb{R}\) exhibits a structured dependence on both the context and the arm, characterized by a Lipschitz continuity condition, consistent with \cite{slivkins2011contextual}. This regularity assumption ensures that small perturbations in the context or arm do not induce arbitrarily large variations in expected rewards. Specifically, we impose an \(L\)-Lipschitz condition on \(\nu\) with respect to the product metric, defined as
\begin{equation}
   \Phi\bigl((a,x),(a',x')\bigr) = \Phi_X(x,x') + \Phi_A(a,a'),
\end{equation}
where \(\Phi_X\) and \(\Phi_A\) denote the respective metrics on the context and arm spaces. Consequently, for all \((a, x), (a', x') \in A \times X\), the reward function satisfies
\begin{equation}
   \bigl|\nu(a,x) - \nu(a',x')\bigr|
   \leq L \cdot \bigl[\Phi_X(x,x') + \Phi_A(a,a')\bigr].
\end{equation}
where \(L > 0\) is a Lipschitz constant that quantifies the maximum rate at which the expected reward can change as a function of variations in the context or arm. The Lipschitz property plays a crucial role in defining the complexity of learning in this setting, as it imposes a smoothness constraint that facilitates the use of discretization techniques for efficient decision-making.

Since \(X\) and \(A\) are metric spaces, their covering dimensions, denoted \(d_x\) and \(d_a\), characterize their intrinsic complexity. The covering dimension quantifies the asymptotic scaling of the covering number, governing the granularity of finite approximations. As the metric on \(A \times X\) is the sum of \(\Phi_X\) and \(\Phi_A\), standard results yield \(d = d_x + d_a\).

Solving the contextual Lipschitz bandit problem in continuous spaces is intractable due to their infinite dimensionality. To overcome this, we approximate the product space \(A \times X\) with a finite \(\psi\)-covering. We define finite sets \(X_0 \subseteq X\) and \(A_0 \subseteq A\) such that, for each context \(x \in X\), there exists \(x_0 \in X_0\) with \(\Phi_X(x, x_0) \leq \psi\), and for each arm \(a \in A\), there exists \(a_0 \in A_0\) with \(\Phi_A(a, a_0) \leq \psi\). This discretization ensures a finite representative set, facilitating efficient approximation of the continuum.

Let the armed played by the algorithm at time $t$ be denoted by $a_t$ for the input context $x_t$, and \(a^*(x_t) \in \arg\max_{a\in A} \nu(a, x_t)\) be an optimal arm for context \(x_t\). The instantaneous regret at round \(t\) is \(\nu(a^*(x_t), x_t) - \nu(a_t, x_t),\) yielding the expectation of the cumulative regret $\mathcal{R}_T(A)$ for the arms set $A$ and time horizon $T$ as, 
\begin{align}
    \mathbb{E}[\mathcal{R}_T(A)] = \sum_{t=1}^T \Bigl[ \nu(a^*(x_t), x_t) - \mathbb{E}[\nu(a_t, x_t)] \Bigr].
\end{align}
In contextual bandits, a myopic agent may favor immediate rewards, neglecting exploration. To mitigate this, the principal introduces a compensation term \(\kappa_t\) to incentivize exploratory actions, modifying the observed reward to \(r_t = \rho_t(a_t) + \gamma_t(\kappa_t)\), where \(\gamma_t(\cdot)\) is an increasing function modeling the agent’s response to compensation. The compensation is set as \(\kappa_t = \hat{\nu}_t(g_t, x_t) - \hat{\nu}_t(a_t, x_t),\) where \(g_t\) is the agent’s greedy choice based on empirical estimates \(\hat{\nu}_t\). 

This modifies the principal’s objective to minimizing both regret and total compensation cost, given by \(\mathcal{C}_T(A) = \sum_{t=1}^T \kappa_t\). We establish sublinear bounds on both regret and incentive cost by leveraging the discretization scale \(\psi\), covering dimension \(d\), and Lipschitz constant \(L\). The following sections formalize these results and optimal strategies for balancing exploration and incentives.

\section{Incentivized Exploration with Uniform Discretization on Infinite Arms} \label{sec:uniform}
In this section, we introduce an incentivized exploration framework where a principal selects actions from a continuous space while ensuring effective exploration through a uniform discretization scheme. The principal employs a compensation mechanism to align the agent’s choices with the desired exploration strategy, balancing regret minimization and incentive costs. We analyze this approach in both stochastic (Algorithm \ref{meta-inc-uniform}) and contextual (Algorithm \ref{meta-inc-contextual}) bandit settings, providing theoretical guarantees on expected regret and compensation.\\

\noindent\textbf{Incentivized Exploration with Stochastic Rewards.} 
\begin{algorithm}[t]
\caption{Incentivized Exploration in Infinite-Armed Stochastic Bandits}
\label{meta-inc-uniform}
\begin{algorithmic}[1]
\REQUIRE Metric space for arms \((A, \Phi)\), time horizon \(T\), Lipschitz constant \(L\)
\STATE Set: \(\psi \gets O\left(T^{-1/(d+2)} L^{-2/(d+2)} (\log T)^{1/(d+2)}\right)\)
\STATE Construct a \(\psi\)-covering \(A_0 \subseteq A\)
\FOR{\(t = 1,2,\dots,T\)} 
    \STATE Select arm: \(a_t \gets \arg\max_{a \in A_0} \hat{\mu}_t(a) + \sqrt{\frac{2\log t}{N_t(a)}}\)
    \STATE Agent picks: \(g_t \gets \arg\max_{a \in A_0} \hat{\mu}_t(a)\)
    \STATE Compensation: \(\kappa_t \gets \hat{\mu}_t(g_t) - \hat{\mu}_t(a_t)\)
    \STATE Pull \(a_t\), observe reward: \(r_t = \rho_t(a_t) + \gamma_t(\kappa_t)\)
    \STATE Update: \(\hat{\mu}_{t+1}(a_t) \gets \text{Update based on } r_t\)
\ENDFOR
\end{algorithmic}
\end{algorithm}
Algorithm~\ref{meta-inc-uniform} addresses the problem of incentivizing exploration in stochastic bandit settings over a compact metric space \((A, \Phi)\). The algorithm begins by computing a discretization parameter \(\psi = O\left(T^{-1/(d+2)} L^{-2/(d+2)} (\log T)^{1/(d+2)}\right)\), which defines a finite \(\psi\)-covering subset \(A_0 \subseteq A\) (lines 1–2). This discretization enables tractable learning in an otherwise infinite action space.

At each round \(t\), the principal selects an arm \(a_t \in A_0\) based on a UCB rule (line 4): \(a_t \gets \arg\max_{a \in A_0} \hat{\mu}_t(a) + \sqrt{2 \log t / N_t(a)}\), where \(\hat{\mu}_t(a)\) is the empirical mean reward and \(N_t(a)\) the count of prior selections. This balances exploration and exploitation by favoring uncertain yet promising arms. Meanwhile, the agent selects \(g_t = \arg\max_{a \in A_0} \hat{\mu}_t(a)\) (line 5), purely maximizing empirical reward. Since \(g_t\) may differ from the principal’s choice, the principal offers compensation \(\kappa_t = \hat{\mu}_t(g_t) - \hat{\mu}_t(a_t)\) (line 6), ensuring the agent remains indifferent between \(a_t\) and their preferred arm.

Upon receiving the incentive, the agent pulls \(a_t\), observing a drifted reward \(r_t = \rho_t(a_t) + \gamma_t(\kappa_t)\) (line 7), where \(\rho_t(a_t)\) is the stochastic reward and \(\gamma_t(\kappa_t)\) models the compensation’s effect. The principal then updates \(\hat{\mu}_{t+1}(a_t)\) (line 8).

By repeating this process over \(T\) rounds, the algorithm effectively aligns incentives while ensuring efficient exploration. Its performance is characterized by the expected regret \(\mathbb{E}[\mathcal{R}_T]\) and total compensation \(\mathbb{E}[\mathcal{C}_T]\), as analyzed in the following section.

\begin{theorem}
\label{thm-reg-ucb}
Consider Algorithm~\ref{meta-inc-uniform} run on a metric space \((A, \Phi)\), where arms are discretized using a \(\psi\)-covering \(A_0 \subseteq A\) satisfying \(|A_0| \leq O(1 / \psi^d)\), with \(d\) denoting the covering dimension of \((A, \Phi)\). If \(\psi = O\left(T^{-1/(d+2)} L^{-2/(d+2)} (\log T)^{1/(d+2)}\right)\), then the expected regret and compensation satisfy:
\[
\mathbb{E}[\mathcal{R}_T] = \Tilde{O}\left( L^{\frac{d}{d+2}} T^{\frac{d+1}{d+2}} \right), \quad 
\mathbb{E}[\mathcal{C}_T] = \Tilde{O}\left( L^{\frac{d}{d+2}} T^{\frac{d+1}{d+2}} \right),
\]
where \(L\) is the Lipschitz constant of the reward function \(\mu(a)\), and the hidden constants in the \(O(\cdot)\) and \(\Tilde{O}(\cdot)\) notations depend on the drift parameter \(\ell\).
\end{theorem}
These bounds indicate that both regret and compensation grow sublinearly, ensuring the principal’s cost remains manageable while achieving effective exploration. The regret bound aligns with known results for bandit settings with covering constraints, confirming that our incentivized approach retains efficiency despite the additional compensation mechanism. \\

\noindent\textbf{Incentivized Exploration with Contextual Information.} 
\begin{algorithm}[t]
\caption{Incentivized Exploration in Infinite-Armed Contextual Bandits}
\label{meta-inc-contextual}
\begin{algorithmic}[1]
\REQUIRE Metric spaces \((A, \Phi_A), (X, \Phi_X)\), snapping function \(\xi: X \to X_0\), horizon \(T\)
\STATE Set: \(\psi \gets O\left(T^{-1/(d+2)} L^{-2/(d+2)} (\log T)^{1/(d+2)}\right)\)
\STATE Construct \(\psi\)-covering \(A_0 \times X_0 \subset A \times X\)
\FOR{\(t = 1,2,\dots,T\)}
    \STATE Observe context \(x_t \in X\), snap: \(x_0 \gets \xi(x_t)\)
    \STATE Principal picks arm:
    \[
    a_t \gets \arg\max_{a_0 \in A_0} \hat{\nu}_t(a_0, x_0) + \sqrt{\frac{2\log t}{N_t(a_0, x_0)}}
    \]
    \STATE Agent picks: \(g_t \gets \arg\max_{a_0 \in A_0} \hat{\nu}_t(a_0, x_0)\)
    \STATE Compensation: \(\kappa_t \gets \hat{\nu}_t(g_t, x_0) - \hat{\nu}_t(a_t, x_0)\)
    \STATE Reward: \(r_t = \rho_t(a_t) + \gamma_t(\kappa_t)\)
    \STATE Update: \(\hat{\nu}_{t+1}(a_t, x_0) \gets \text{Update based on } r_t\)
\ENDFOR
\end{algorithmic}
\end{algorithm}
Algorithm~\ref{meta-inc-contextual} addresses the problem of incentivized learning in a contextual bandit setting, where the principal must guide exploration while contending with an agent who acts based on personal utility. The algorithm begins by computing a discretization parameter \(\psi = O\left(T^{-1/(d+2)} L^{-2/(d+2)} (\log T)^{1/(d+2)}\right)\), which is used to construct a finite \(\psi\)-covering \(A_0 \times X_0 \subset A \times X\) (lines 1–2). This covering provides a structured discretization of the joint action-context space, allowing the principal to approximate the continuous setting with a finite grid for efficient decision-making.

At each round \(t\), the environment reveals a context \(x_t \in X\) (line 4). The principal maps \(x_t\) to a nearby representative point \(x_0 = \xi(x_t) \in X_0\) using a predefined snapping function \(\xi\). Given this snapped context, the principal selects an arm \(a_t \in A_0\) according to a UCB rule (line 5): \(a_t \gets \arg\max_{a_0 \in A_0} \hat{\nu}_t(a_0, x_0) + \sqrt{2 \log t / N_t(a_0, x_0)}\), where \(\hat{\nu}_t(a_0, x_0)\) is the empirical mean reward estimate and \(N_t(a_0, x_0)\) denotes the number of times that pair has been selected.

Concurrently, the agent selects an arm \(g_t = \arg\max_{a_0 \in A_0} \hat{\nu}_t(a_0, x_0)\) based solely on empirical estimates (line 6), choosing the arm that currently appears most rewarding. Since \(g_t\) may not match the principal’s chosen arm \(a_t\), a compensation \(\kappa_t = \hat{\nu}_t(g_t, x_0) - \hat{\nu}_t(a_t, x_0)\) is offered (line 7). This ensures the agent is indifferent between its own choice and the principal’s intended action, maintaining strategic alignment. After receiving the compensation, the agent pulls arm \(a_t\) and observes the reward \(r_t = \rho_t(a_t) + \gamma_t(\kappa_t)\) (line 8), where \(\rho_t(a_t)\) is the true stochastic reward, and \(\gamma_t(\kappa_t)\) captures the effect of monetary incentives. The principal then updates the empirical estimate \(\hat{\nu}_{t+1}(a_t, x_0)\) using the observed reward (line 9). Through this process, the algorithm effectively balances exploration and exploitation while ensuring incentive compatibility. Its performance is analyzed in terms of expected cumulative regret \(\mathbb{E}[\mathcal{R}_T]\) and total compensation \(\mathbb{E}[\mathcal{C}_T]\), as formalized in the following theorem.

\begin{theorem}
\label{thm-reg-ctx}
Consider Algorithm~\ref{meta-inc-contextual} executed on a metric space \((A \times X, \Phi)\), where a \(\psi\)-covering discretization \(A_0 \times X_0 \subseteq A \times X\) is used with \(|A_0 \times X_0| \leq O(1 / \psi^d)\), and \(d = d_a + d_x\) denotes the covering dimension of \((A \times X, \Phi)\). If \(\psi = O\left(T^{-1/(d+2)} L^{-2/(d+2)} (\log T)^{1/(d+2)}\right)\), then the expected regret and compensation satisfy:
\[
\mathbb{E}[\mathcal{R}_T] = \Tilde{O}\left( L^{\frac{d}{d+2}} T^{\frac{d+1}{d+2}} \right), \quad 
\mathbb{E}[\mathcal{C}_T] = \Tilde{O}\left( L^{\frac{d}{d+2}} T^{\frac{d+1}{d+2}} \right),
\]
where \(L\) is the Lipschitz constant of the reward function \(\nu(a, x)\), and the hidden constants in the \(O(\cdot)\) and \(\Tilde{O}(\cdot)\) notations depend on the drift parameter \(\ell\).
\end{theorem}

\section{Numerical Results} \label{app:num_res} 

We evaluate our proposed incentivized exploration algorithms in infinite-armed bandit settings with reward drift, focusing on two regimes: \emph{stochastic rewards} (Algorithm~\ref{meta-inc-uniform}) and \emph{stochastic rewards with contextual information} (Algorithm~\ref{meta-inc-contextual}). In both cases, the arm space is modeled as a compact metric space \( A \subseteq [0,1]^d \) with covering dimension \( d \in \{1,2,3\} \), and the reward function \(\mu(a)\) is Lipschitz continuous with constant \(L \in \{d_1^c, d_2^c, d_3^c\}\). We simulate 10 independent trials over a time horizon \(T = 20{,}000\). At each round, the agent selects arms greedily based on empirical rewards, while the principal chooses arms using a UCB-based policy and offers compensation to align incentives.

The observed reward at time \(t\) is modeled as \( r_t = \rho_t(a_t) + \gamma_t(\kappa_t) \), where \(\rho_t(a_t)\) is the stochastic payoff and \(\gamma_t(\kappa_t) = \ell_t \cdot \kappa_t\) captures the reward drift induced by compensation. Specifically, we sample \(\rho_t(a_t) \sim \mathcal{N}(\mu(a_t), 0.05)\), where the mean reward function is chosen to be linear: \(\mu(a) = L \cdot \sum_{i=1}^d a_i\). This implies optimal rewards \(\mu^* = L, 2L, 3L\) for \(d = 1,2,3\), respectively. The drift factor \(\ell_t\) is drawn uniformly in \([0.45, 0.55]\) to simulate mild variability in agent sensitivity, corresponding to a drift function \(\gamma(\kappa_t) \approx \kappa_t / 2\). In contextual experiments, the context \(x_t \in [0,1]^d\) is drawn uniformly at each round, and rewards are generated via a joint Lipschitz function \(\nu(a, x)\) with identical noise and drift models. To construct \(\nu(a, x)\), we use a separable linear function of the form \(\nu(a, x) = L \cdot (\sum_i a_i + \sum_j x_j)\), ensuring Lipschitz continuity with respect to both arms and contexts.

We report average regret and compensation along with 95\% confidence intervals. In the stochastic setting, Algorithm~\ref{meta-inc-uniform} discretizes the arm space via a theoretically guided \(\psi\)-covering, and both regret and compensation scale as \( \Tilde{O}(T^{(d+1)/(d+2)}) \) (see combined results in Figure ~\ref{fig:main-plots}(a) and ~\ref{fig:main-plots}(b)). For the contextual setting, Algorithm~\ref{meta-inc-contextual} discretizes both arms and contexts, and empirical performance similarly aligns with theoretical bounds (Figure~\ref{fig:main-plots}(c) and ~\ref{fig:main-plots}(d)). In both regimes, regret and compensation increase with the problem dimension but remain within the predicted theoretical upper bounds, indicated by the dotted lines.

\begin{figure}[h]
    \centering
    \includegraphics[width=0.5\textwidth]{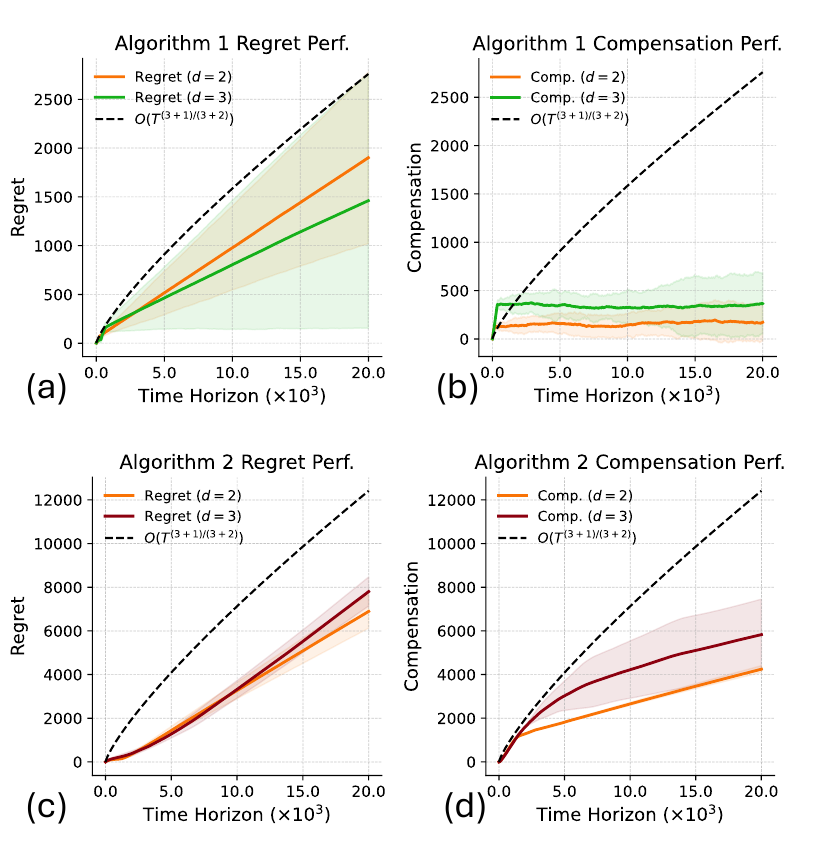}
    \caption{Performance of Algorithm~1 ((a), (b)) and Algorithm~2 ((c), (d)) in terms of average regret and compensation over a horizon of 20{,}000. The dotted lines indicate the theoretical upper bound corresponding to the highest value of $d$.}
    \label{fig:main-plots}
\end{figure}

To further illustrate the importance of setting the discretization parameter \(\psi\) optimally, we evaluate the performance of Algorithm~\ref{meta-inc-uniform} under suboptimal choices of \(\psi\) for various dimensions. As shown in Table~\ref{suboptimal-psi}, deviations from the theoretically optimal \(\psi\) (either too small or too large) can significantly degrade performance. A small \(\psi\) results in an excessively large discretized set \(A_0\), which increases computational burden and sharply inflates compensation due to over-exploration. Conversely, a large \(\psi\) leads to under-coverage of the action space and elevated regret. These results highlight the sensitivity of the algorithm to discretization granularity and empirically validate our choice of scaling \(\psi = O\left(T^{-1/(d+2)} L^{-2/(d+2)} (\log T)^{1/(d+2)}\right)\), where the constant factor was selected through a light hyperparameter search.

\begin{table}
  \caption{Impact of suboptimal \(\psi\) values on the discretized arm set size \(|A_0|\), regret, and compensation. Arrows indicate under- or over-estimation of \(\psi\).}

  \label{suboptimal-psi}
  \centering
  \begin{tabular}{lllll}
    \toprule
    $d$     & $\psi$     & $|A_0|$ & Regret & Comp. \\
    \midrule
    1 & 0.061 (opt) & 17 & 222.95 & 44.60 \\
    1 & 0.00001 ($\downarrow$) & 100000 & 6662.64 & 3821.73 \\
    1 & 0.35 ($\uparrow$) & 3 & 2773.73 & 51.78 \\
    \midrule
    2 & 0.123 (opt) & 81 & 184.25 & 106.27 \\
    2 & 0.005 ($\downarrow$) & 40000 & 1344.90 & 9911.4 \\
    2 & 0.35 ($\uparrow$) & 9 & 2769.69 & 56.64 \\
    \midrule
    3 & 0.187 (opt) & 216 & 659.33 & 240.29 \\
    3 & 0.02 ($\downarrow$) & 125000 & 1173.66 & 9790.02 \\
    3 & 0.35 ($\uparrow$) & 27 & 2775.99 & 72.81 \\
    \bottomrule
  \end{tabular}
\end{table}

\section{Lower Bounds and Adaptiveness}

So far, our algorithms have adopted the naive strategy of uniformly discretizing the space of arms $A$ to obtain a ``good" approximation of the underlying arm space. Another well-known approach towards discretizing the arm space involves adaptively focusing on areas of the metric space that contain arms close to optimal (called the `zooming' algorithm) (\cite{Kleinberg-metric}). Once the arms that are under consideration are chosen adaptively, the algorithm uses a UCB-like strategy to pick an arm to play, i.e., pick the arm with the highest upper confidence bound. Thus, we focus on comparing our UCB-based algorithm with zooming since they are the most similar. 

The zooming algorithm for arbitrary metric spaces can achieve $\tilde{O}(T^{\frac{d_z+1}{d_z+2}})$ (\cite{Kleinberg-lipschitz}) {regret} where $d_z$ is the \lq zooming dimension\rq. Intuitively, unlike the covering dimension, which aims at covering the entire metric space, the zooming dimension focuses on covering a subset ${\tilde{A}} \subset A$ where {$\tilde{A}$} contains near-optimal arms (we refer the reader to \cite{Kleinberg-metric} for a technical definition). {Obviously}, the zooming dimension can never be larger than the covering dimension, and in the worst case where near-optimal arms are spread out over the metric space, the covering dimension equals the zooming dimension (i.e. $d = d_z$). Since our bounds are over \emph{all} metric spaces, a zooming-based strategy for discretization would get the same regret upper bound as {Theorem} \ref{thm-reg-ucb} in the worst case and thus, in general, an adaptive discretization strategy does not help improve regret bounds.

Note that existing regret bounds discussed above apply exclusively to the standard Lipschitz bandit problem without incentivization or reward drift. In contrast, our results incorporate both incentivization and reward drift alongside the Lipschitz assumption. Our findings show that, even with these additional complexities, the carefully chosen incentive structure achieves regret upper bounds matching those from scenarios without incentives - while \emph{simultaneously} maintaining sublinear compensation.

Nevertheless, there remains an important open question: Could adaptive discretization further improve regret guarantees in specific problem instances (beyond worst-case scenarios)? If so, it is also unclear whether these improved regret bounds could still be combined with sublinear compensation costs. We leave the detailed exploration of this potential trade-off to future research.

Finally, there exist lower bounds on regret for the Lipschitz bandit problem: no algorithm can do better than $\tilde{O}(T^{\frac{d_m+1}{d_m+2}})$ where $d_m$ is the Max-Min Covering dimension (\cite{Kleinberg-lipschitz}). Intuitively, the Max-Min covering dimension tries to quantify the isometry between two $\epsilon$-balls in the same space and is less than or equal to the covering dimension in general. In the worst case, it is equal to the covering dimension when we have a homogenous metric space (refer to \cite{Kleinberg-lipschitz} for a technical definition) i.e. $d = d_m$. Therefore, similar to the zooming case, when {an arbitrary} metric space {is} considered, we obtain an upper bound that matches the lower bound. Thus, we need instance-specific, structural assumptions on the metric space to prove improved bounds via uniform discretization or a more sophisticated algorithm.

\section{Conclusion}  
In this paper, we introduced incentivized exploration algorithms for infinite-armed multi-armed bandit problems with reward drift, addressing critical real-world scenarios where decision-makers must encourage exploration among myopic agents. By employing a uniform discretization approach and carefully designed incentive structures, our algorithms achieve simultaneous sublinear regret and compensation guarantees, both scaling as  in terms of the covering dimension . We also successfully extended our framework to contextual bandits, yielding analogous theoretical and empirical performance. These results highlight the feasibility and effectiveness of incentivized exploration in high-dimensional continuous settings. Future directions include exploring adaptive discretization strategies to potentially improve instance-specific regret bounds and further generalizing incentive mechanisms for dynamic and adversarial environments.


\bibliographystyle{IEEEtran}
\bibliography{refs}

\section{appendix}
This document is organized as follows. Section \ref{app:back} provides a summary of the notation used throughout the paper, along with additional background on key concepts and formal definitions where necessary. Section \ref{app:applications} outlines real-world applications of our framework. Section \ref{app:proofs} contains detailed proofs for all results stated in the main paper. Finally, Section \ref{app:num_res} offers empirical validation of our algorithms and results, supporting the theoretical findings with numerical experiments.

\section{Background and Notation} \label{app:back}

In this section, we provide additional background to the concepts used in our paper along with their formal definitions as used in prior work for completeness. We begin by summarizing the notation in our paper.

\subsection{Summary of Notation}
A summary of notation used throughout the paper is given in Table \ref{tab:notation}.
\begin{table*}[t]
    \centering
    \renewcommand{\arraystretch}{1}  
    \begin{tabular}{p{3.5cm} p{10cm}}
        \toprule
        \textbf{Symbol} & \textbf{Definition} \\
        \midrule
        
        $A$ & Set of arms. \\

        $X$ & Set of contexts. \\

        $\Phi_A$ & Metric on the arm space. \\

        $\Phi_X$ & Metric on the context space. \\

        $\Phi$ & Product metric on $A \times X$, defined as $\Phi((a, x), (a', x')) = \Phi_A(a, a') + \Phi_X(x, x')$. \\

        $\mu(a)$ & Expected reward of arm $a$. \\

        $\nu(a, x)$ & Expected reward of arm $a$ given context $x$. \\

        $\rho_t(a)$ & Instantaneous reward of arm $a$ at time $t$. \\

        $\rho_t(a, x)$ & Instantaneous reward of arm $a$ given context $x$ at time $t$. \\

        $\gamma_t(\cdot)$ & Drift function at time $t$. \\

        $\kappa_t$ & Compensation at time $t$. \\

        $r_t$ & Observed reward at time $t$. \\

        $\hat{\mu}_t(a)$ & Empirical mean reward of arm $a$ at time $t$. \\

        $\hat{\nu}_t(a, x)$ & Empirical mean reward of arm $a$ given context $x$ at time $t$. \\

        $g_t$ & Greedy arm choice at time $t$. \\

        $a_t$ & Principal’s recommended arm at time $t$. \\

        $\mathcal{R}_T$ & Cumulative regret up to time $T$. \\

        $\mathcal{C}_T$ & Cumulative compensation up to time $T$. \\

        $d$ & Covering dimension of the metric space. \\

        $d_a$ & Covering dimension of the arm space $(A, \Phi_A)$. \\

        $d_x$ & Covering dimension of the context space $(X, \Phi_X)$. \\

        $d_z$ & Zooming dimension. \\

        $\psi$ & Mesh used for uniform discretization. \\

        $N_t(a)$ & Number of pulls of arm $a$ up to time $t$. \\

        $N_t(a, x)$ & Number of times arm $a$ was pulled for context $x$ up to time $t$. \\

        $L$ & Lipschitz constant of the reward function. \\

        $\ell_t$ & Lipschitz constant of the drift function at time $t$. \\

        $\ell$ & Maximum Lipschitz constant of the drift function. \\

        $\tilde{A}$ & Subset of near-optimal arms. \\

        $\mu'(a)$ & Reward function incorporating incentivization. \\

        $\nu'(a, x)$ & Contextual reward function incorporating incentivization. \\

        $\gamma(y)$ & Drift function based on compensation paid. \\

        $\xi(x)$ & Snapping function that maps context $x$ to the nearest point in $X_0$. \\

        $X_0$ & Discretized set of contexts used in uniform discretization. \\

        $A_0$ & Discretized set of arms used in uniform discretization. \\

        $A_0 \times \mathcal{X}_0$ & Discretized arm-context space used in contextual bandits. \\

        \bottomrule
    \end{tabular}
    \caption{Summary of Notation Used in the Paper}
    \label{tab:notation}
\end{table*}

\subsection{Metric Spaces}

A metric space provides a rigorous mathematical framework for quantifying distances, defining neighborhoods, and analyzing convergence in a given set of objects. It plays a fundamental role in numerous mathematical disciplines, including topology, functional analysis, and optimization, and has significant applications in machine learning, online decision-making, and multi-armed bandits (\cite{lai1985asymptotically, Kleinberg-lipschitz, slivkins2011contextual}). 

Formally, a metric space is defined as a pair \((\mathcal{X}, \Phi)\), where \(\mathcal{X}\) is a set, and \(\Phi: \mathcal{X} \times \mathcal{X} \to \mathbb{R}\) is a function known as a \textit{metric} or \textit{distance function}. The function \(\Phi\) assigns a non-negative real value to each pair of elements in \(\mathcal{X}\), which intuitively represents the "distance" between them. To be a valid metric, \(\Phi\) must satisfy the following fundamental axioms (\cite{rudin1964principles, Kleinberg-lipschitz}):

\begin{enumerate}
    \item \textbf{Non-negativity and Identity of Indiscernibility:}  
    The function \(\Phi\) is always non-negative, ensuring that distances are meaningful in a real-valued setting:
    \[
    \Phi(x, y) \geq 0, \quad \forall x, y \in \mathcal{X},
    \]
    with equality if and only if \(x = y\), i.e., the only way for the distance between two elements to be zero is if they are identical:
    \[
    \Phi(x, y) = 0 \iff x = y.
    \]
    
    \item \textbf{Symmetry:}  
    The metric is symmetric, meaning that the distance between two points remains unchanged when their order is reversed:
    \[
    \Phi(x, y) = \Phi(y, x), \quad \forall x, y \in \mathcal{X}.
    \]
    This property ensures that \(\Phi\) is well-defined as a measure of pairwise separation.

    \item \textbf{Triangle Inequality:}  
    The metric satisfies the triangle inequality, which states that for any three points \(x, y, z \in \mathcal{X}\), the direct path between \(x\) and \(z\) is never longer than the sum of the paths via an intermediate point \(y\):
    \[
    \Phi(x, z) \leq \Phi(x, y) + \Phi(y, z), \quad \forall x, y, z \in \mathcal{X}.
    \]
    This property guarantees that the metric behaves consistently with intuitive geometric notions of distance, preventing shortcuts or violations of spatial consistency.
\end{enumerate}

These three axioms collectively ensure that the metric function \(\Phi\) properly encodes the structure of \(\mathcal{X}\) in a way that enables rigorous reasoning about proximity, separation, and continuity. Classical examples of metric spaces include the Euclidean space \(\mathbb{R}^d\) with the \(\ell_p\) norm (such as \(\ell_2\), the standard Euclidean distance), discrete metric spaces, and more general structures like Riemannian manifolds (\cite{munkres2013topology, lattimore_2020}).

\subsection{Covering Numbers} \label{cov-num}

In the study of metric spaces and statistical learning theory, \textit{covering numbers} serve as a fundamental tool for quantifying the complexity of a space (\cite{mendelson2003entropy, vershynin2018high}). Covering numbers provide a measure of how well a given space can be approximated using a finite set of representative points, making them particularly useful in the analysis of function approximation, discretization techniques, and generalization bounds in machine learning (\cite{anthony2009neural, lattimore_2020}).  

Given a metric space \((\mathcal{X}, \Phi)\), the covering number, denoted by \(\mathcal{N}(x, \mathcal{X}, \Phi)\), represents the minimum number of metric balls of radius \(x\) required to cover the entire space \(\mathcal{X}\). The metric \(\Phi\) determines the notion of "distance" in \(\mathcal{X}\) and thereby governs the size of each covering ball. Covering numbers are widely used in fields such as functional analysis, empirical process theory, and bandit learning, particularly in settings involving infinite or high-dimensional spaces (\cite{bubeck2011lipschitz, Kleinberg-lipschitz}).

\paragraph{Formal Definition.}  
For a subset \(S \subseteq \mathcal{X}\) of a metric space \((\mathcal{X}, \Phi)\), the \(x\)-covering number is defined as the smallest number of balls of radius \(x\) needed to cover \(S\):

\begin{align*}
    \mathcal{N}(x, &S, \Phi) = \\
    &\min \left\{ n \in \mathbb{N} : \exists \{x_i\}_{i=1}^{n} \subseteq \mathcal{X}, \quad S \subseteq \bigcup_{i=1}^{n} B(x_i, x) \right\},
\end{align*}

where \(B(x_i, x)\) denotes the closed metric ball centered at \(x_i\) with radius \(x\), i.e.,  

\begin{equation}
    B(x_i, x) = \{ y \in \mathcal{X} \mid \Phi(y, x_i) \leq x \}.
\end{equation}

Intuitively, the covering number quantifies how many discrete points are necessary to approximate the entire space within a given resolution \(x\). Spaces with small covering numbers at fine resolutions exhibit lower complexity, whereas high-dimensional or irregular spaces tend to have larger covering numbers.

\paragraph{Scaling Behavior and Dimensional Dependence.}  
The growth rate of covering numbers is closely tied to the intrinsic dimension of the metric space. In a \(d\)-dimensional Euclidean space \((\mathbb{R}^d, \ell_2)\), it is well known that the covering number scales polynomially as:

\begin{equation} 
    \mathcal{N}(x, \mathbb{R}^d, \ell_2) = O(x^{-d}).
\end{equation}

More generally, for a compact metric space \((\mathcal{X}, \Phi)\) with covering dimension \(d\), the covering number satisfies:

\begin{equation} \label{eq:disc}
    \mathcal{N}(x, \mathcal{X}, \Phi) = O(x^{-d}),
\end{equation}

which establishes a fundamental link between metric entropy and the geometric structure of \(\mathcal{X}\) (\cite{vershynin2018high, Kleinberg-lipschitz}).

\subsection{\(\psi\)-Covers of Metric Spaces} \label{psi-cov-back}

A \(\psi\)-cover provides a structured means of discretizing a continuous metric space while controlling the resolution of the approximation. Such covers are widely used in topology, metric geometry, and statistical learning theory to quantify complexity and facilitate numerical approximations (\cite{mendelson2003entropy, vershynin2018high}). In the context of multi-armed bandits and online learning, \(\psi\)-covers play a critical role in uniform discretization schemes, ensuring that continuous spaces can be efficiently explored with a finite representative set (\cite{Kleinberg-metric, slivkins2011contextual}).

\paragraph{Formal Definition.}  
Let \((X, d)\) be a metric space, and let \(\psi > 0\). A \(\psi\)-cover of \(X\) is a collection of subsets \(\{U_i\}_{i \in I}\) such that:

\begin{itemize}
    \item \(X = \bigcup_{i \in I} U_i\), ensuring that every point in \(X\) is contained in at least one subset \(U_i\).
    \item Each subset \(U_i\) has diameter at most \(\psi\), i.e.,
    \[
    \text{diam}(U_i) = \sup\{d(x, y) : x, y \in U_i\} \leq \psi, \quad \forall i \in I.
    \]
\end{itemize}

The notion of a \(\psi\)-cover is closely related to covering numbers \(\mathcal{N}(\psi, X, d)\), which quantify the minimal number of such subsets required to cover \(X\) (\cite{anthony2009neural, lattimore_2020}). When applied in algorithmic settings, a \(\psi\)-cover serves as the basis for constructing finite approximations of infinite or high-dimensional spaces.

\paragraph{Application in Uniform Discretization.}  
In our work, all uniform discretization algorithms (Sections (\ref{prob-form}) and (\ref{sec:uniform})) utilize a \(\psi\)-mesh, which is constructed by selecting a representative point from each subset in the covering. This process effectively reduces a continuous optimization problem to a discrete one while preserving the underlying metric structure. The quality of such a discretization directly impacts the trade-off between approximation error and computational efficiency in multi-armed bandits and reinforcement learning (\cite{bubeck2011lipschitz, Kleinberg-lipschitz}).

\subsection{Technical Lemmas}

\begin{lemma}[Hoeffding's Inequality] \label{hoeff}
Let \(X_1, X_2, \dots, X_n\) be independent random variables satisfying \(X_i \in [a_i, b_i]\) almost surely. Define the sample mean as  
\[
\bar{X} = \frac{1}{n} \sum_{i=1}^{n} X_i,
\]
with expectation \(\mu = \mathbb{E}[\bar{X}]\). Then, for any \(t > 0\), the following concentration bound holds:
\begin{equation}
\Pr(|\bar{X} - \mu| \geq t) \leq 2\exp\left(-\frac{2n^2t^2}{\sum_{i=1}^{n} (b_i - a_i)^2}\right).
\end{equation}
\end{lemma}

Hoeffding's inequality provides an exponential concentration bound on the deviation of the empirical mean from its expectation. In our analysis, we employ a specialized form of this result to bound deviations in empirical reward estimates.

\begin{lemma}[Modified Hoeffding's Inequality] \label{hoeff-mod}
Suppose that rewards are independent and identically distributed (i.i.d.) in the interval \([0,1]\). Let \(\tilde{\mu}(a)\) denote the empirical mean reward for arm \(a\) after \(N\) observations. Define the confidence radius:
\begin{equation}
r(a) = \sqrt{\frac{2 \log T}{N}}.
\end{equation}
Then, with probability at least \(1 - 2/T^{4}\), the following concentration bound holds:
\begin{equation}
|\tilde{\mu}(a) - \mu(a)| \leq r(a).
\end{equation}
\end{lemma}

This result ensures that with high probability, the empirical mean estimate remains within a confidence interval of width \(r(a)\) around the true mean reward \(\mu(a)\). This bound is fundamental in our regret analysis, particularly in controlling the spread of empirical estimates in bandit algorithms.

\section{Applications} \label{app:applications}

In this section, we outline some real-world applications of our framework.

\subsection{E-commerce Product Recommendations}

In e-commerce platforms, the arm space $A$ represents the product catalog, with metric $\Phi$ capturing product similarity. The expected reward $\mu(a)$ corresponds to customer satisfaction or revenue. Customers act as agents, while the platform is the principal. The incentive mechanism $\kappa_t$ represents discounts or coupons. The drift function $\gamma_t(\kappa_t)$ models how incentives inflate customer ratings, reflecting the tendency of incentivized customers to provide more positive feedback. This drift challenges the platform's ability to accurately assess product quality over time.

\subsection{Content Recommendation in Streaming Services}

For streaming platforms, arms represent content items, and $\Phi$ encodes content similarity. The reward $\mu(a)$ measures user engagement. Viewers are agents, and the streaming service is the principal. Incentives $\kappa_t$ could be promotional offers for premium content. The drift $\gamma_t(\kappa_t)$ models how these incentives temporarily boost viewer ratings or engagement metrics, potentially skewing the platform's assessment of content quality. This drift reflects the challenge of distinguishing between genuinely engaging content and content that appears popular due to promotional effects.

\subsection{Dynamic Pricing in Ride-sharing Platforms}

In ride-sharing applications, arms represent pricing strategies, while $\Phi$ captures pricing model similarity. The reward $\mu(a)$ balances customer satisfaction and driver earnings. Riders and drivers are agents, with the platform as principal. Incentives $\kappa_t$ could be price adjustments or bonuses. The drift $\gamma_t(\kappa_t)$ models how these incentives temporarily alter user behavior, such as increased ride acceptance rates or higher ratings, which may not reflect long-term satisfaction. This drift complicates the platform's ability to determine optimal pricing strategies.

\subsection{Adaptive Clinical Trials}

In medical research, arms represent treatments, and $\Phi$ measures treatment similarity. The reward $\mu(a)$ quantifies treatment efficacy. Patients are agents, with researchers as the principal. Incentives $\kappa_t$ could be additional medical support. The drift $\gamma_t(\kappa_t)$ models how these incentives might influence patient-reported outcomes or adherence, potentially inflating the perceived efficacy of certain treatments. This drift reflects the challenge of separating true treatment effects from the effects of additional care or attention given to participants in certain trial arms.

\subsection{Personalized Education Platforms}

For online learning platforms, arms represent educational content, while $\Phi$ captures resource similarity. The reward $\mu(a)$ measures learning outcomes. Students are agents, and the platform is the principal. Incentives $\kappa_t$ could be bonus credits or certificates. The drift $\gamma_t(\kappa_t)$ models how these incentives might temporarily boost student engagement or performance metrics, which may not accurately reflect long-term learning outcomes. This drift challenges the platform's ability to identify truly effective educational content versus content that appears successful due to short-term motivational effects of the incentives.

\section{Proofs} \label{app:proofs}
In this section, we present proofs of all our results.

\subsection{Analysis of Theorem 1 (Incentivized Exploration with Stochastic Rewards)}
Analogous to the previous cases, we begin by analyzing the expected regret. Recall from Section 3 of the main paper that the expected regret of a stochastic bandit algorithm employing an incentivized exploration scheme over the infinite arm set \( A \) (i.e., Algorithm (\ref{meta-inc-uniform})) is given by  
\begin{align*}
    \mathbb{E}\left[\mathcal{R}_T(A)\right] &= T\cdot \mu_{A}^* - \mathbb{E}\left[\sum_{t=0}^T \rho_t(a_t)\right] \\
    &=T\left(\mu_{A}^* -\mu_{A_0}^*\right) + T\cdot \mu_{A_0}^* - \mathbb{E}\left[\sum_{t=0}^T \rho_t(a_t)\right] \\
    &= T\cdot \delta_{A_0} + \mathbb{E}\left[\mathcal{R}_T(A_0)\right],
\end{align*}  
where \( \delta_{A_0} = \mu_{A}^* - \mu_{A_0}^* \) represents the approximation (or the discretization) error, and \( \mathcal{R}_T(A_0) \) denotes the regret of the algorithm on the approximate arm set \( A_0 \) over the time horizon \( T \). From Equation (\ref{lips}) in Section \ref{prob-form}, it follows that \( \delta_{A_0} \leq L\psi \). Thus, the analysis reduces to bounding \( \mathbb{E}[\mathcal{R}_T(A_0)] \).

Let \( a_{\tau} \) be the arm played by the algorithm at time \( \tau \). The corresponding reward drift associated with this arm is given by \( \gamma_{\tau}(\kappa_{\tau}) \). Let \( N_t(a) \) denote the number of times arm \( a \) has been selected up to time \( t \). Additionally, define \( \ell = \max_t \ell_t \), where \( \ell_t \) is the Lipschitz constant of the unknown drift function, as stated in Assumption \ref{drift-assm}. Following \cite{liu20}, the total reward drift accumulated by arm \( a \) over \( t \) time steps is given by  
\begin{align} \label{Gamma}
    \Gamma_t(a) = \sum_{\tau=1}^t \gamma_{\tau}(\kappa_{\tau}) \cdot \mathbf{1} \left(a_{\tau} = a\right) \leq 2\ell \sqrt{2N_t(a) \log T}.
\end{align}  

Next, we analyze the expected regret \( \mathcal{R}_T(A_0) \). To do so, we employ the \textit{clean event} analysis technique (see, e.g., \cite{slivkins2021introduction}). The key idea is to ensure that the empirical estimates of the rewards remain sufficiently close to their true expectations with high probability.  

Let \( \tilde{\mu}_t(a) \) denote the drift-free empirical mean reward of arm \( a \) at time \( t \). That is, \( \tilde{\mu}_t(a) \) can be viewed as the sample mean of i.i.d. rewards drawn from the underlying distribution \( \mathcal{D}(a) \), independent of any drift effects. Our goal is to ensure that \( \tilde{\mu}_t(a) \) remains close to the true expected reward \( \mu(a) \) with high probability.  

By Hoeffding's inequality (see (\ref{hoeff-mod})), we obtain the following probabilistic guarantee for any fixed arm \( a \) and time \( t \):
\begin{align*}
    \mathbb{P} \left[ \left| \tilde{\mu}_t(a) - \mu(a) \right| \leq \sqrt{\frac{2 \log T}{N_t(a)}} \right] \geq 1 - \frac{2}{T^4}.
\end{align*}

We now define the clean event \( \mathcal{E}_{\text{clean}} \) as the event in which the above concentration bound holds for all arms \( a \in A_0 \) and all time steps \( t \leq T \), i.e.,
\begin{align*}
    \mathcal{E}_{\text{clean}} := \left\{ \forall t \leq T, \forall a \in A_0, \left| \tilde{\mu}_t(a) - \mu(a) \right| \leq \sqrt{\frac{2 \log T}{N_t(a)}} \right\}.
\end{align*}
Applying the union bound over all arms and time steps, we obtain
\begin{align} \label{clean-ev-prob}
    \mathbb{P}(\mathcal{E}_{\text{clean}}) \geq 1 - \frac{2 |A_0|}{T^3}.
\end{align}

At any time \( t \), let \( a_t \) be the arm selected by Algorithm (\ref{meta-inc-uniform}). To bound the worst-case regret under the clean event, we observe that the algorithm selects \( a_t \) only if its upper confidence bound (UCB) is at least as large as that of the optimal arm \( a^* \). That is, the selection criterion implies  
\begin{align}
\tilde{\mu}_t(a_t) + \frac{\Gamma_t(a_t)}{N_t(a_t)} 
&+ \sqrt{\frac{2 \log T}{N_t(a_t)}} \nonumber \\[0.5em]
&\geq \tilde{\mu}_t(a^*) + \frac{\Gamma_t(a^*)}{N_t(a^*)} 
+ \sqrt{\frac{2 \log T}{N_t(a^*)}}.
\label{ucb-select}
\end{align}

\textbf{Step 1}: We will expand the individual terms using the clean event. We know that under the clean event \( \mathcal{E}_{\text{clean}} \), for all arms \( a \in A_0 \) and all time steps \( t \leq T \), we have 
\begin{align*}
    \left| \tilde{\mu}_t(a) - \mu(a) \right| \leq \sqrt{\frac{2 \log T}{N_t(a)}}.
\end{align*}
Applying this bound to both \( a_t \) and \( a^* \), we obtain  
\begin{align*}
    \mu(a_t) &\geq \tilde{\mu}_t(a_t) - \sqrt{\frac{2 \log T}{N_t(a_t)}} \\
     &\text{and} \\
    \mu(a^*) &\leq \tilde{\mu}_t(a^*) + \sqrt{\frac{2 \log T}{N_t(a^*)}}.
\end{align*}

\textbf{Step 2:} The instantaneous regret incurred by selecting \( a_t \) is defined as \(\theta(a_t) = \mu(a^*) - \mu(a_t)\). Substituting the clean event bounds on \( \mu(a^*) \) and \( \mu(a_t) \), we get  
\begin{align*}
    \mu(a^*) - \mu(a_t)
    &\leq \tilde{\mu}_t(a^*) + \sqrt{\frac{2 \log T}{N_t(a^*)}} - \left( \tilde{\mu}_t(a_t) - \sqrt{\frac{2 \log T}{N_t(a_t)}} \right) \\
    &= \left( \tilde{\mu}_t(a^*) - \tilde{\mu}_t(a_t) \right) + \sqrt{\frac{2 \log T}{N_t(a_t)}} + \sqrt{\frac{2 \log T}{N_t(a^*)}}.
\end{align*}
Rearranging the UCB selection rule (\ref{ucb-select}), we obtain  
\begin{align*}
    \tilde{\mu}_t(a^*) - \tilde{\mu}_t(a_t)
    &\leq \frac{\Gamma_t(a_t)}{N_t(a_t)} - \frac{\Gamma_t(a^*)}{N_t(a^*)} + \sqrt{\frac{2 \log T}{N_t(a_t)}} - \sqrt{\frac{2 \log T}{N_t(a^*)}} \\
    &\leq \frac{\Gamma_t(a_t)}{N_t(a_t)} + \sqrt{\frac{2 \log T}{N_t(a_t)}}.
\end{align*}

Using (\ref{Gamma}), we get the bound \( \Gamma_t(a) / N_t(a) \leq 2\ell \sqrt{2 \log T / N_t(a)} \), valid under the clean event assumption, we conclude that
\begin{align}\label{inst-regret-ucb}
    \theta(a_t) = \mu(a^*) - \mu(a_t) \leq 2(\ell + 1) \sqrt{\frac{2\log T}{N_t(a_t)}}.
\end{align}

Thus, the instantaneous regret for selecting any suboptimal arm \( a_t \in A_0 \setminus \{a^*\} \) is bounded by a term that decreases as \( N_t(a_t) \) increases, ensuring improved estimates over time.

Under the clean event and (\ref{inst-regret-ucb}), the total regret \( \mathcal{R}_T(A_0) \) after \( T \) time steps is given by  
\begin{align*}
    \mathcal{R}_T(A_0) = \sum_{a \in A_0} \theta(a)\cdot N_T(a) \leq 2(\ell + 1) \sqrt{2\log T} \sum_{a \in A_0} \sqrt{N_T(a)}.
\end{align*}
Since \( \sqrt{x} \) is a concave function for \( x \in \mathbb{R}_{\geq 0} \) and the total number of arm pulls satisfies \(\sum_{a \in A_0} N_T(a) = T\), we apply Jensen's inequality to obtain
\begin{align*}
    \frac{1}{|A_0|}\sum_{a \in A_0} \sqrt{N_T(a)} \leq \sqrt{\frac{1}{|A_0|} \sum_{a \in A_0} N_T(a)} = \sqrt{\frac{T}{|A_0|}}.
\end{align*}

Multiplying both sides by \( |A_0| \), we get  
\begin{align*}
    \sum_{a \in A_0} \sqrt{N_T(a)} \leq |A_0| \sqrt{\frac{T}{|A_0|}} = \sqrt{T |A_0|}.
\end{align*}
Substituting this bound into the regret expression, we obtain  
\begin{align}\label{reg-ucb}
    \mathcal{R}_T(A_0) \leq 2\sqrt{2}(\ell + 1) \sqrt{T |A_0| \log T}.
\end{align}
We now analyze the expected regret \( \mathbb{E}[\mathcal{R}_T(A_0)] \) by considering the probability of the clean event and its complement. Since the probability of the clean event, derived earlier using the union bound (i.e., (\ref{clean-ev-prob})), is at least
\begin{align*}
    \mathbb{P}(\mathcal{E}_{\text{clean}}) \geq 1 - \frac{2 |A_0|}{T^3},
\end{align*}
the probability of its complement, i.e., the event where the clean event does not hold, is at most \( 2 |A_0| / T^3 \). In the worst case, the regret in this scenario can be at most \( T \), since each action could contribute at most unit regret per round. Thus, by conditioning on these two events and using (\ref{reg-ucb}), we obtain
\begin{align*}
    &\mathbb{E}\left[\mathcal{R}_T(A_0)\right] \\
    &= \mathbb{E}\left[\mathcal{R}_T(A_0) \mid \mathcal{E}_{\text{clean}}\right] \mathbb{P}(\mathcal{E}_{\text{clean}}) 
    + \mathbb{E}\left[\mathcal{R}_T(A_0) \mid \mathcal{E}_{\text{clean}}^c\right] \mathbb{P}(\mathcal{E}_{\text{clean}}^c) \\
    &\leq \mathcal{R}_T(A_0) + (2|A_0|/T^3) \cdot T.
\end{align*}
Next, we apply \ref{eq:disc}, which states that for uniform discretization, the number of arms in \( A_0 \) is bounded by \(|A_0| \leq \lambda_d/\psi^d\), for some constant \( \lambda_d > 0 \). Substituting this bound into the regret expression, we obtain
\begin{align*}
    \mathbb{E}\left[\mathcal{R}_T(A_0)\right] &\leq 2\sqrt{2}(\ell + 1) \sqrt{T \log T \cdot \frac{\lambda_d}{\psi^d}} + \frac{2\lambda_d}{\psi^d T^2} \\
    &= 2\sqrt{2}(\ell + 1) \sqrt{\frac{T \log T}{\psi^d} \lambda_d} + \frac{2\lambda_d}{\psi^d T^2}
\end{align*}
Since the last term is negligible for large \( T \), we introduce a constant \( \lambda_1 > 0 \) to simplify notation, yielding
\begin{align*}
    \mathbb{E}\left[\mathcal{R}_T(A_0)\right] \leq 2 \lambda_1\sqrt{2\lambda_d} (\ell + 1) \sqrt{\frac{T \log T}{\psi^d}}.
\end{align*}
Finally, combining this result with the definition of \( \delta_{A_0} \), which satisfies \( \delta_{A_0} \leq L\psi \), the expected regret for Algorithm (\ref{meta-inc-uniform}) on the infinite arm set \( A \) is given by
\begin{align} \label{fin-reg-ucb}
    \mathbb{E}\left[\mathcal{R}_T(A)\right] \leq T L\psi + 2 \lambda_1 \sqrt{2\lambda_d} (\ell + 1) \sqrt{\frac{T \log T}{\psi^d}}.
\end{align}

This bound highlights the trade-off between approximation error \( L\psi \) and the regret incurred on the discretized arm set \( A_0 \), demonstrating that the regret scales sublinearly in \( T \) while also depending on the granularity of the parameter \( \psi \).

To minimize the expected regret, we select the optimal discretization parameter \( \psi \) by minimizing the leading-order term in the regret bound from Equation (\ref{fin-reg-ucb}). Setting \(\psi = \left[8\lambda_1^2\lambda_d\left(\ell + 1\right)^2 T^{-1} L^{-2} (\log T)\right]^{1/(d+2)}\), we obtain the final regret bound.

Next, we analyze the total compensation, denoted by \( \mathcal{C}_T(A) \), which satisfies \( \mathcal{C}_T(A) = \mathcal{C}_T(A_0) \). This equivalence holds because compensation, unlike regret, is only incurred when the chosen arm deviates from the principal's recommended choice, both of which are restricted to the same set \( A_0 \). 

Let \( \hat{\mu}_t(a) \) denote the empirical average of the observed (drifted) rewards for an arm \( a \in A_0 \) up to time \( t \). Compensation is incurred at time \( t \) when the principal's recommended arm \( a_t \) is chosen over the greedy arm \( g_t \), yet their empirical reward estimates satisfy
\begin{align} \label{comp}
    \hat{\mu}_t(a_t) + \sqrt{\frac{2 \log T}{N_t(a_t)}} \geq \hat{\mu}_t(g_t) + \sqrt{\frac{2 \log T}{N_t(g_t)}}.
\end{align}
Rearranging, we obtain the compensation incurred at time \( t \),
\begin{align*}
    \kappa_t = \hat{\mu}_t(g_t) - \hat{\mu}_t(a_t) \leq \sqrt{\frac{2 \log T}{N_t(a_t)}}.
\end{align*}

To further bound the compensation, we use Equation (\ref{inst-regret-ucb}) along with Equation (\ref{lips}) in Section \ref{prob-form} of the main paper to establish a lower bound on the number of times an arm is pulled:
\begin{align} \label{n-lb}
    N_t(a_t) \geq \frac{8(\ell+1)^2 \log T}{L^2 \psi^2}.
\end{align}
Substituting this into the upper bound on \( \kappa_t \), we obtain
\begin{align}\label{kappa-ucb-ub}
    \kappa_t \leq \frac{L\psi}{\sqrt{2}(\ell+1)}.
\end{align}
The total expected compensation is given by summing over all time steps,
\begin{align*}
    \mathbb{E}\left[\mathcal{C}_T(A_0)\right] = \sum_{t=1}^{T} \kappa_t \leq \frac{T L\psi}{\sqrt{2}(\ell+1)}.
\end{align*}
Finally, substituting the optimal value of \( \psi \) from Equation (\ref{fin-reg-ucb}), we obtain the final bound on the total expected compensation.

\subsection{Analysis of Theorem 2 (Uniform Discretization of Contextual Lipschitz Bandits) with Reward Drift}
\noindent

\textbf{Problem Setting.}
We consider a contextual bandit framework in which each round \(t = 1,2,\dots,T\) unfolds as follows. A context \(x_t \in X \subseteq [0,1]^{d_x}\) is revealed to the learner, where \(X\) is equipped with a metric \(\Phi_X\). Based on this observed context, the learner then chooses an arm \(a_t \in A \subseteq [0,1]^{d_a}\), where \(A\) is equipped with a metric \(\Phi_A\). Finally, a random reward \(r_t\) is observed, with conditional expectation \(\nu(a_t, x_t)\). We assume \(X\) and \(A\) are compact metric spaces, making it possible to apply standard covering arguments in our subsequent analysis.

\vspace{1em}
\noindent
\textbf{Lipschitz Reward Function.}
We assume that the unknown reward function \(\nu: A \times X \to \mathbb{R}\), consistent with \cite{slivkins2011contextual}, is \(L\)-Lipschitz with respect to the product metric
\[
   \Phi\bigl((a,x),(a',x')\bigr) 
   \;=\; 
   \Phi_X(x,x') + \Phi_A(a,a').
\]
Formally, for all \((a,x),(a',x') \in A \times X\),
\[
   \bigl|\nu(a,x)\;-\;\nu(a',x')\bigr|
   \;\;\le\;
   L \,\bigl[\Phi_X(x,x') \;+\; \Phi_A(a,a')\bigr].
\]
We denote by \(d_x\) the covering dimension of \(X\) under \(\Phi_X\) and by \(d_a\) the covering dimension of \(A\) under \(\Phi_A\). Consequently, the product space \(A \times X\) endowed with \(\Phi\) has covering dimension \(d = d_a + d_x\). 

\vspace{1em}
\noindent
\textbf{Regret Definition and Decomposition.}
Let \(a^*(x_t) \in \arg\max_{a\in A} \nu(a,x_t)\) be an optimal arm for the observed context \(x_t\). Define the instantaneous regret on round \(t\) by
\[
   \max_{a\in A}\,\nu\bigl(a,x_t\bigr)
   \;-\;
   \nu\bigl(a_t,x_t\bigr).
\]
Summing these over the \(T\) rounds and taking expectations, we obtain the cumulative expected regret:
\[
   \mathbb{E}[\mathcal{R}_T]
   \;=\;
   \sum_{t=1}^T
     \Bigl[
        \max_{a\in A}\,\nu\bigl(a,x_t\bigr)
        \;-\;
        \mathbb{E}\bigl[\nu(a_t,x_t)\bigr]
     \Bigr].
\]
The proposed incentivized exploration algorithm, designed for continuous sets \(X\) and \(A\), operates by uniformly discretizing these spaces. Specifically, we define finite \(\psi\)-nets \(X_0 \subseteq X\) and \(A_0 \subseteq A\), ensuring that each \(x \in X\) has a corresponding \(x_0 \in X_0\) within \(\psi\) under \(\Phi_X\) (i.e., \(\Phi_X(x, x_0) \leq \psi\)), and similarly, each \(a \in A\) has a nearest \(a_0 \in A_0\) within \(\psi\) under \(\Phi_A\). To facilitate this discretization, we introduce a ``snapping'' function \(\xi: X \to X_0\), which maps each observed context \(x_t\) to its closest grid point \(\xi(x_t)\). The learner then selects arms exclusively from \(A_0\) based on these discretized contexts, effectively reducing the original infinite decision space to a finite one:
\begin{align*}
   &\mathbb{E}[\mathcal{R}_T] \\
   &\;=\;
   \underbrace{\sum_{t=1}^T 
      \Bigl[\max_{a\in A}\,\nu\bigl(a,x_t\bigr)
            \;-\;
            \max_{a_0\in A_0}\,\nu\bigl(a_0,\xi(x_t)\bigr)
      \Bigr]}_{\text{Discretization Error}} \\
   &\;+\;
   \underbrace{\sum_{t=1}^T 
      \Bigl[\max_{a_0\in A_0}\,\nu\bigl(a_0,\xi(x_t)\bigr)
            \;-\;
            \mathbb{E}\bigl[\nu(a_t,x_t)\bigr]
      \Bigr]}_{\text{Finite-Bandit Regret}}.
\end{align*}
The first term in this decomposition captures the inherent approximation gap incurred by working with the discrete subsets \((X_0,A_0)\), while the second term reflects the standard multi-armed bandit (or contextual bandit) regret over a finite set of actions \(A_0\), conditioned on discrete contexts from \(X_0\).

\vspace{1em}
\noindent
\textbf{Contextual UCB via Incentivized Exploration.}
To solve the bandit problem on the discretized space \(X_0 \times A_0\), the algorithm (Algorithm~\ref{meta-inc-contextual}) maintains a separate instance of a UCB-based incentivized exploration procedure for each discrete context \(x_0 \in X_0\). That is, upon observing context \(x_t\), the algorithm snaps it to its nearest grid point \(x_0 = \xi(x_t)\), and then invokes the UCB routine for that specific \(x_0\) to select an arm \(a_t \in A_0\). This modular treatment allows efficient reuse of reward estimates and confidence intervals across rounds for each discrete context and ensures that regret accumulates independently per context.

\vspace{1em}
\begin{lemma}[Discretization Error]
\label{lem:discretization_error}
Suppose \(\nu\) is \(L\)-Lipschitz with respect to \(\Phi\). For any \(\psi\)-nets \(X_0 \subseteq X\) and \(A_0 \subseteq A\), let \(\xi(x_t)\) be the point in \(X_0\) nearest to \(x_t\). Then, summing over \(T\) rounds, the total discretization error satisfies:
\[
   \sum_{t=1}^T
      \Bigl[\max_{a\in A}\,\nu\bigl(a,x_t\bigr)
            \;-\;
            \max_{a_0\in A_0}\,\nu\bigl(a_0,\xi(x_t)\bigr)\Bigr]
   \;\le\;
   2\,L\,T\,\psi.
\]
\end{lemma}

\begin{proof}
In each round \(t\), let \(x_t \in X\) and define \(a^*(x_t) \;\in\; \arg\max_{a \in A}\,\nu\bigl(a,x_t\bigr)\).
Since \(A_0 \subseteq A\) is a \(\psi\)-net under \(\Phi_A\), there is an \(a_0^*(x_t)\in A_0\) such that \(\Phi_A\!\bigl(a^*(x_t),\,a_0^*(x_t)\bigr)\le \psi.\) Also, \(\xi(x_t)\) is the nearest point in \(X_0\subseteq X\) to \(x_t\) under \(\Phi_X,\) hence \(\Phi_X\!\bigl(x_t,\,\xi(x_t)\bigr)\le\psi.\) By the \(L\)-Lipschitz property of \(\nu\) with respect to both \(\Phi_A\) and \(\Phi_X\), we obtain:
\[
   \nu\bigl(a^*(x_t),\,x_t\bigr)
   \;\le\;
   \nu\bigl(a_0^*(x_t),\,x_t\bigr) + L\psi,
\]
\[
   \nu\bigl(a_0^*(x_t),\,x_t\bigr)
   \;\le\;
   \nu\bigl(a_0^*(x_t),\,\xi(x_t)\bigr) + L\psi.
\]
Combining these yields 
\[
   \nu\bigl(a^*(x_t),\,x_t\bigr)
   \;\le\;
   \nu\bigl(a_0^*(x_t),\,\xi(x_t)\bigr) + 2L\psi.
\]
Since \(\max_{a \in A}\,\nu(a,x_t) = \nu\bigl(a^*(x_t), x_t\bigr)\) and 
\(\max_{a_0 \in A_0}\,\nu(a_0, \xi(x_t)\bigr) \,\ge\, \nu\bigl(a_0^*(x_t), \xi(x_t)\bigr),\)
it follows that 
\begin{align}\label{disc-err-inst}
   \max_{a \in A}\,\nu(a,x_t)
   \;-\;
   \max_{a_0 \in A_0}\,\nu\bigl(a_0,\xi(x_t)\bigr)
   \;\le\;
   2L\psi. 
\end{align} 
Summing equation (\ref{disc-err-inst}) over \(t = 1,\dots,T\) gives us the required expression.
\end{proof}

\vspace{1em}
\noindent
\textbf{Bounding the Finite-Bandit Regret.}
In order to obtain a complete bound on \(\mathbb{E}[\mathcal{R}_T]\), one must add the regret arising from the discretization step to the regret accumulated by the finite-bandit procedure that operates on the discrete sets \(X_0\subseteq X\) and \(A_0\subseteq A\). From Lemma \ref{lem:discretization_error} we get that the discretization induces an error of \(2LT\psi\).

Meanwhile, for each discrete context \(x_0\in X_0\), the incentivized exploration algorithm (Algorithm~\ref{meta-inc-contextual}) applies Algorithm~\ref{meta-inc-uniform} to the restricted arm set \(A_0\), maintaining separate statistics per context. That is, a separate instance of the UCB policy is run for each \(x_0\), treating the problem as a standard multi-armed bandit over \(A_0\), with observations and updates only when \(\xi(x_t) = x_0\). 

From Theorem (\ref{thm-reg-ucb}), the regret for each context is bounded by:
\[
  2\sqrt{2}\,(\ell+1)\,\sqrt{\,T_{x_0}\,\lvert A_0\rvert\,\log T}\,,
\]
where \(T_{x_0}\) is the number of rounds in which the snapped context was \(x_0\). Since \(\sum_{x_0\in X_0} T_{x_0} = T\), summing over all contexts yields:
\begin{align*}
  \sum_{x_0\in X_0}
    2\sqrt{2}\,(\ell+1)\,&\sqrt{T_{x_0}\,\lvert A_0\rvert\,\log T}
  \;=\; \\
  &2\sqrt{2}\,(\ell+1)\,\sqrt{\,\lvert A_0\rvert\,\log T\,}
  \;\sum_{x_0\in X_0}\,\sqrt{T_{x_0}}.
\end{align*}
Applying Jensen’s inequality:
\[
  \sum_{x_0\in X_0}\,\sqrt{T_{x_0}}
  \;\le\;
  \sqrt{\,\lvert X_0\rvert\;\sum_{x_0\in X_0} T_{x_0}}
  \;=\;
  \sqrt{\,\lvert X_0\rvert\,T\,}.
\]
Thus, the finite-bandit regret is at most:
\[
2\sqrt{2}\,(\ell+1)\,
  \sqrt{\,\lvert A_0\rvert\,\lvert X_0\rvert\,T\,\log T}.
\]
With uniform discretization, the sizes of \(A_0\) and \(X_0\) are bounded as:
\[
\lvert A_0\rvert \leq \lambda_a \psi^{-d_a}, \quad
\lvert X_0\rvert \leq \lambda_x \psi^{-d_x},
\]
hence:
\[
\lvert A_0\rvert \cdot \lvert X_0\rvert \leq \lambda_x \lambda_a \psi^{-d},
\quad\text{where } d = d_x + d_a.
\]
Substituting into the regret bound:
\[
\mathbb{E}[\mathcal{R}_T] \leq 2L T \psi + 2\sqrt{2}(\ell + 1) \sqrt{\frac{\lambda_x \lambda_a T \log T}{\psi^d}}.
\]
To minimize this upper bound, we choose the optimal value of \(\psi\) as:
\[
\psi = \Bigl[ \sqrt{2 \lambda_a \lambda_x}(\ell + 1) L^{-1} \Bigr]^{2/(d+2)} 
T^{- 1/(d+2)} (\log T)^{1/(d+2)}.
\]
Substituting back yields the final regret bound:
\[
\mathbb{E}[\mathcal{R}_T] 
\;\leq\;
C(\lambda_a, \lambda_x, \ell)\, L^{d/(d+2)}\, 
T^{(d+1)/(d+2)}\, (\log T)^{1/(d+2)},
\]
with
\begin{align*}
&C(\lambda_a, \lambda_x, \ell) = \\
&2 \Bigl( \sqrt{2 \lambda_a \lambda_x} (\ell + 1) \Bigr)^{\frac{2}{(d+2)}} 
\Bigl[ 1 + \sqrt{2} (\ell+1) (\lambda_a \lambda_x)^{\frac{1}{2} - \frac{d}{(2(d+2))}} \Bigr].
\end{align*}

\paragraph{Compensation.}
The compensation structure in the contextual setting remains analogous to the stochastic case. Specifically, once a context \(x_t\) is snapped to its nearest grid point \(\xi(x_t) = x_0\), the algorithm invokes the incentivized UCB subroutine on the finite arm set \(A_0\). The compensation paid at each round is governed by the deviation from greedy choice under empirical drifted rewards, and—as established in Equation (\ref{kappa-ucb-ub})—is upper bounded by:
\[
\kappa_t \leq \frac{L\psi}{\sqrt{2}(\ell+1)}.
\]
This bound holds for all contexts, since the compensation depends only on the arm-space resolution and reward smoothness, not on the context itself. Therefore, the total expected compensation across \(T\) rounds is bounded by:
\[
\mathbb{E}[\mathcal{C}_T] \leq \frac{T L \psi}{\sqrt{2}(\ell+1)}.
\]
Substituting the optimal \(\psi\), we get the final compensation bound that scales as:
\[
\mathbb{E}[\mathcal{C}_T] = \Tilde{O}\left(T^{\frac{d+1}{d+2}}\right).
\]

\end{document}